\documentclass{l4dc2024}

\usepackage{times}
\usepackage{grffile}
\usepackage{amsmath,amssymb,amsbsy}
\usepackage{hyperref}
\usepackage{color}
\usepackage{comment}
\usepackage{lipsum}
\usepackage[normalem]{ulem}
\usepackage{enumerate}
\usepackage{epstopdf}
\usepackage{bbm}
\usepackage{tikz}

\newtheorem{assumption}{Assumption}

\title[Smoothening Effect of Noise in the Reverse Process of SGMs]{Noise in the reverse process improves the approximation \\ capabilities of diffusion models}
\usepackage{times}

\author{%
	\Name{Karthik Elamvazhuthi}$^1$ \Email{karthike@ucr.edu}\\
	\Name{Samet Oymak}$^{2}$ \Email{oymak@umich.edu}\\
	\Name{Fabio Pasqualetti}$^1$ \Email{fabiopas@engr.ucr.edu}\\
	\addr 1. 
	University of California, Riverside \\
	\addr 2. University of Michigan
}

\begin{document}
	
	\maketitle
	
	\begin{abstract}%
		In Score based Generative Modeling (SGMs), the state-of-the-art in generative modeling, stochastic reverse processes are known to perform better than their deterministic counterparts. This paper delves into the heart of this phenomenon, comparing neural ordinary differential equations (ODEs) and neural stochastic differential equations (SDEs) as reverse processes. We use a control theoretic perspective by posing the approximation of the reverse process as a trajectory tracking problem. We analyze the ability of neural SDEs to approximate trajectories of the Fokker-Planck equation, revealing the advantages of stochasticity. First, neural SDEs exhibit a powerful regularizing effect, enabling $L^2$ norm trajectory approximation surpassing the Wasserstein metric approximation achieved by neural ODEs under similar conditions, even when the reference vector field or score function is not Lipschitz. Applying this result, we establish the class of distributions that can be sampled using score matching in SGMs, relaxing the Lipschitz requirement on the gradient of the data distribution in existing literature. Second, we show that this approximation property is preserved when network width is limited to the input dimension of the network. In this limited width case, the weights act as control inputs, framing our analysis as a controllability problem for neural SDEs in probability density space. This sheds light on how noise helps to steer the system towards the desired solution and illuminates the empirical success of stochasticity in generative modeling.
	\end{abstract}
	

	\section{INTRODUCTION}

	Generative modeling, an important tool in machine learning, addresses the challenge of drawing new samples from an unknown data distribution when provided with a set of data samples. Among the cutting-edge techniques in this domain are Score Based Generative Models (SGMs) \cite{ho2020denoising,song2020score,yang2022diffusion}, which hold promising applications in diverse fields such as medical imaging \cite{chung2022score}, path planning \cite{yang2022diffusion}, and shape generation \cite{zhou20213d}. 
	
	SGMs employ two distinct processes: a forward process that gradually transforms the data distribution into a noise distribution, and a reverse process that retraces the trajectories of the forward process in reverse, effectively mapping the noise distribution back to the data distribution. Multiple choices exist for the reverse process. One can opt for a deterministic reverse process, known as the probabilistic flow ODE, or a stochastic reverse process. It has been observed in practice that \cite{song2020score} that stochastic reverse process performs better than the deterministic one. This paper delves into the implications of choosing between these two approaches from the standpoint of trajectory approximation. The two possible choices of reverse processes are two instances of neural ODEs and SDEs respectively. This papers examines the ability neural SDEs to approximate trajectories of vector-fields arising in SGMs. We additionally look at the capability of limited width networks, to emphasize the role of stochasticity in improving approximation capabilities in generative modeling. 
	
	The approximation capabilities of limited width networks have been a subject of independent research since the problem becomes equivalent to a controllability problem. Hence, one can use control theoretic tools to understand approximation properties of neural networks. For instance, \cite{tabuada2022universal} explored the approximation capabilities of limited width residual neural networks (Resnets) from the perspective of control theory, focusing on constructing approximations of maps. Similarly, \cite{ruiz2023neural,ruiz2023control} showed the universal approximation properties of limited width neural ODEs for density estimation. In \cite{elamvazhuthi2022neural}, it is shown that one can control the continuity equation corresponding to the neural ODE to the solution of a continuity equation of an arbitrary Lipschitz vector-field, in the Wasserstein metric.
	
	We use a similar control theoretic methodology in this paper for the case of neural SDEs. We show that using neural SDEs one can approximate trajectories of Fokker-Planck equation uniformly in the $L^2$ norm. Of special interest is the case of Fokker-Planck equations arising from score matching algorithm of SGMs. In this context, we characterize the class of distributions that can be sampled with this approach, achieving a stronger norm approximation than previously possible with deterministic models. This implies that neural SDEs offer improved sampling capabilities as the reverse process in SGMs. While \cite{tzen2019theoretical} establish general universal approximation properties of neural SDEs, they do not consider the general trajectory approximation problem, as relevant to score matching. Additionally, their results relied on arbitrary network widths and approximation in the weaker Kullback–Leibler divergence divergence, rather than the $L^2$ norm. Furthermore, we eliminate the assumption of Lipschitz continuity on the data distribution's gradients, previously imposed by \cite{tzen2019theoretical}, and others analyzing SGM sampling capabilities \cite{chen2023sampling,de2022convergence}. By considering diffusion processes on bounded domains, we are able to additionally allow for data distributions with compact connected supports. This expands the range of data distributions that can be sampled using neural SDEs.
	
	The two main contributions of the paper can be summarized as follows
	i) Characterizing the class of distributions that can achieve $L^2$ norm closeness between densities with score approximation  ii) Universal Approximation capabilities of Neural SDEs for the same distributions for the case with limited width.

	The paper is organized in the following way. In Section \ref{sec:prob}, we formulate the problem addressed in this paper. In Section \ref{sec:not} we present some notation and assumptions for the results presented in the paper later. In Section \ref{sec:analysis}, we present the controllability analysis and our main results on universal approximation with extra regularity. 
	
	\section{Problem Formulation and Motivation}
	\label{sec:prob}

	SGM is a technique to draw new samples from a data distribution based on available samples. In this section, we reinterpret SGM from a control theoretic perspective and then present the problem addressed in this paper. The SGM algorithm involves two differential equations (SDEs),  as presented in \cite{song2020score}. First, we have the {\it forward process}, which is a stochastic differential equation, for example, 
	\begin{eqnarray} 
	\label{eq:sde}
	dX = g(X)dt + \sqrt{2} d \mathbf{W}+d\psi(t)\nonumber\\
	X_0 \sim \rho_d 
	\end{eqnarray}
	where $\mathbf{W}(t)$ is the standard $d$-dimensional Brownian motion and $\psi$ is a stochastic process that ensures that the process remains confined to some domain $\Omega \subset \mathbb{R}^d$ with boundary $\partial \Omega$. The vector field $g:\mathbb{R}^d \rightarrow \mathbb{R}^d$ is chosen such that the probability density of the random variable $X_t$ converges to a distribution $\rho_{n}$, from which one can easily sample and referred to as the {\it noise distribution}. This is guaranteed by analysing the behavior of the Fokker Planck equation which governs the evolution of $\rho_t$ given by
	\begin{eqnarray}
	\label{eq:fwdpdf}
	&\frac{\partial \rho}{\partial t} = \Delta \rho -\nabla \cdot (g(x) \rho) ~  &\mbox{on} ~[0,T] \times \Omega  \\ \nonumber
	&\vec{n}(x)\cdot ( \nabla\rho_t(x) +g(x)) = 0 ~~~ &\mbox{on} ~\partial \Omega\\ \nonumber 
	&\rho_0 = \rho_d ~~~  &\mbox{on} ~\Omega.
	\end{eqnarray}
	where $\vec{n}(x)$ is the unit vector normal to the boundary of the domain $\partial \Omega$. 
	This boundary condition ensures that  $\int_{\mathbb{R}^d} \rho_t(x)dx = 1$ for all $t\geq 0$.  For the choice $g(x) = -\nabla \log \rho_{n}$ we can show that $\lim_{ t \rightarrow \infty } \rho_t = \rho_{n}$. Since the domain $\Omega$ is bounded we can choose the noise distribution to be the uniform distribution, in which case $f \equiv 0$. Then we can rewrite the above equation as
	\begin{eqnarray}
	\label{eq:probodepdf}
	&\frac{\partial \rho^f}{\partial t} = \Delta \rho^f  = \nabla \cdot  ([\nabla \log \rho^f ] \rho^f) ~  &\mbox{on} ~[0,T] \times \Omega \nonumber \\
	&\vec{n}(x)\cdot  \nabla\rho^f(t,x)  = 0 ~~~ &\mbox{on} ~\partial \Omega \nonumber \\
	&\rho^f(0) = \rho_d ~~~  &\mbox{on} ~\Omega
	\end{eqnarray}
	
	In order to sample from $\rho_d$, fixing $T>0$, one can sample from the noise distribution $\rho_{n}$ and run the {\it reverse process} that approximately has the same probability distribution as $\rho^f_{T-t}$. There are multiple possible choices of the reverse process. One candidate choice, is the {\it probabilistic flow ODE} given by,
	\begin{eqnarray} 
	\label{eq:probode}
	dX_r = \nabla  \log \rho^f_{T-t}dt +d\psi(t)  \nonumber \\
	X_0 \sim \rho_{n}
	\end{eqnarray}
	where $\rho_n$ is the noise distribution that is easy to sample from, and the density of $X_r \sim \rho^r_t = \rho^f_{T-t}$, evolving according to,
	\begin{eqnarray}
	\label{eq:revpde}
	&\frac{\partial \rho^r}{\partial t} =-([\nabla \log \rho^r ] \rho^r) 
	\end{eqnarray}
	However, in practice one does not have complete information about the {\it score}: $\nabla \log \rho^f_{T-t}$. Therefore,  a neural network $s(t,x,\theta)$ with weight parameters $\theta$ is used to approximate this quantity.
	
	This objective is to ensure that the solution $\rho^\theta_t$ of the equation 
	\begin{eqnarray}
	\frac{\partial \rho^\theta_t }{\partial t} = \nabla \cdot  ( [ s(t,x,\theta) ] \rho^\theta_t ) \nonumber \\
	\rho^\theta_0 = \rho_n \nonumber
	\end{eqnarray}
	is close to $\rho^f_{T-t}$ so that we can sample from $\rho_d$ by running the approximating reverse ODE,
	\begin{eqnarray} 
	dX = -s(t,X,\theta)dt +d\psi(t) \nonumber \\
	X_0 \sim \rho_{n}
	\end{eqnarray}
	The neural network $s(t,x,\theta)$ used to approximate the score is identified by solving the the optimization problem,
	\begin{eqnarray} 
	\label{eq:scorlos}
	\min_{\theta}  & &\int_0^T \mathbb{E}_{\rho^f_{T-t}} |s(t,\cdot,\theta) -\nabla \log \rho^f_{T-t} |^2dt \nonumber \\
	&=& \int_0^T \int_{\Omega}|s(t,x,\theta) -\nabla \log \rho^f_{T-t} |^2\rho^f_{T-t}(x)dxdt 
	\end{eqnarray}
	The choice of the reverse process is not unique. In fact, one can choose the alternative stochastic reverse process which as the same density evolution as \eqref{eq:revpde}.
	\begin{eqnarray} 
	\label{eq:probsde}
	dX_r = 2\nabla \log \rho^f_{T-t}dt +\sqrt{2} d\mathbf{W} + d\psi(t)\nonumber \\
	X_0 \sim \rho_{n}
	\end{eqnarray}
	In this case the approximating reverse process used to sample from $\rho_d$ is then given by
	\begin{eqnarray} 
	\label{eq:approxrevsde}
	dX_r = -s(t,x,\theta)dt +\sqrt{2 } d\mathbf{W} +d\psi(t) \nonumber \\
	X_0 \sim \rho_{n}
	\end{eqnarray}
	
	One can view the above problem as a trajectory tracking problem from the perspective of control theory. We have a reference trajectory $\rho^f_t$ in space of probability densities, generated by the equation \eqref{eq:probodepdf}, and we can think of the neural network  $\theta$ as control parameters that are to be tuned to make the solution of \eqref{eq:probodepdf}.
	There are two questions that the above choice of reverse processes: i) What is the benefit of noise for trajectory tracking when comparing \eqref{eq:probode} and \eqref{eq:probsde} ii) What can we say about the expressivity capabilities? 
	
	We first consider these two questions for general approximating classes of functions. Then, we consider these two questions for a particular choice of $s(t,x,\theta)$. When the right-hand side of the neural SDE has {\it limited width}, that is the number of input nodes for each time have width equal to the dimension of the sample space. Toward this end,
	let $\sigma :\mathbb{R} \rightarrow \mathbb{R}$ be a given {\it activation function}. 
	Let the function $\Sigma : \mathbb{R} \rightarrow \mathbb{R}$ be given by
	\[\Sigma(x) = [\sigma(x_1),...,\sigma(x_d)]^T\]
	We consider the associated {\it neural stochastic differential equation}, defined by 
	\begin{eqnarray}
	dX= A(t)\Sigma(W(t)X+B(t)) + \sqrt{2}d \mathbf{W}(t)+d\psi
	\label{eq:nsde}
	\end{eqnarray}
	where $A: [0,T] \rightarrow \mathbb{R}^{d \times d}$, $W: [0,T] \rightarrow  \mathbb{R}^d$ and  $B : [0,T] \rightarrow \mathbb{R}^d$ are the control inputs or weights for the neural network. 
	Suppose that the initial condition $x(0)$ of the neural ODE \eqref{eq:node} is random and represented by a probability density function $\rho_0$, that is, $P(X(0) \in \Omega) = \int_{\Omega} \rho_0(x)dx$. Then the uncertainty in the location is $x(t)$ is given by time dependent probability density $\rho_t$ which evolves according to the Fokker-Planck equation
	
	\begin{eqnarray}
	\label{eq:neurtra}
	\frac{\partial \rho}{\partial t} = \Delta \rho - \nabla \cdot \big(v_t(x)\rho \big)   \\ \nonumber
	\rho_0= \rho_0 \\ \nonumber
	v_t(x)=(A(t)\Sigma(W(t)x+\theta(t)))
	\end{eqnarray}

	The control problem that we address in this paper is the following: {\it 
		Given a curve on the set of probability densities $t  \mapsto \rho_t$ that is the solution of the following Fokker-Planck equation 
		\begin{eqnarray}
		\label{eq:genPde}
		\frac{\partial \rho}{\partial t} = \Delta \rho -\nabla \cdot \big(V_t(x)\rho \big)    
		\end{eqnarray}
		can we construct weights $A(t),W(t), b(t)$ such that the solution of \eqref{eq:neurtra} is arbitrary close to $\rho_t$ in a suitable sense, for all $t \in [0,T]$?}
	
	Note that the trajectory tracking problem pertaining to SGMs is a special instance of the above problem for the choice $V = \nabla \log \rho = \frac{\nabla \rho}{\rho}$. For deterministic case, this problem has been affirmatively addressed in \cite{elamvazhuthi2022neural}, where it was shown that this approximation can be achieved in the {\it Wasserstein metric}, in which case the analysis is performed on properties of the {\it neural ODE} 
	\begin{eqnarray}
	\dot{x}(t) = A(t)\Sigma(W(t)x+B(t))
	\label{eq:node}
	\end{eqnarray}
	and instead of the Fokker-Planck equation one studies the approximation properties of the continuity equation
	\begin{eqnarray}
	\label{eq:ctyeq}
	\frac{\partial \rho}{\partial t} =-\nabla \cdot \big(V_t(x)\rho \big)    
	\end{eqnarray}
	which is a general instance of equation \eqref{eq:probodepdf} in SGMs.
	In this paper we consider the approximation capabilities of neural SDEs which have noise in the process, and show that one can in fact achieve approximation in $L^2$ norm, which is stronger than the Wasserstein metric. Moreover, this can be achieved using the same procedure as used in the deterministic case. Hence, showing that noise improves the approximation capability of neural SDEs. The Kullback-Leibler (KL) divergence is another measure of distance between probability measures that is commonly used in practice. However the KL divergence is also known to be weaker than the $L^2$ norm, and in fact, we have (see \cite{gibbs2002choosing}) the following chain of inequalities
	
	\begin{equation}
	W_1(p,q) \leq C' KL(p,q) \leq C \|p-q\|^2_2
	\end{equation}
	for all probability measures $p$ and $q$ that have a common support, for some constant $C,C'>0$. Here $W_1(\cdot,\cdot)$ denotes the $1-$Wasserstein distance between probability measures, $KL(\cdot, \cdot)$ denotes the KL divergence, and $\|\cdot\|_2^2$ denotes the $L^2$ norm between functions. 
	
	To clarify why we cannot achieve approximation of solutions in the deterministic case \eqref{eq:ctyeq}, using standard neural network approximation objectives, in the $L^2$ norm we consider the simple one dimensional case. In this situation, the solution of
	\eqref{eq:ctyeq} can be represented by $\rho_t(x) = \rho(\phi^{-1}_{V,t}(x))|\partial_x \phi^{-1}_{V,t}(x))|$, where $\phi_{V,t}$ is the flow map corresponding to $V$.  Therefore, even if the distance between an approximating vector-field $V^n$ and $V$ is small in $L^2$ (or even stronger $L^{\infty}$ norm), there is no way to control the  
	value of $\rho^n_t(x)$ which depends on the derivative $\partial_x \phi^{-1}_{V^n,t}(x)$, which in turn depends on the Lipschitz constant of $V^n_t$. Therefore, the approximating densities can behave in a very irregular way in the $L^2$ norm. However, as we will show in this paper, this problem does not arise in the stochastic case, due to the regularization effect of noise. One can in fact, relax the Lipschitz assumption on the reference vector fields as well. Moreover, the convergence can be achieved in an average senese in the $H^1$ norm.
	\section{Notation and Preliminaries}
	\label{sec:not}
	We now define some mathematical terms that are used in later sections.  We define $L^2(\Omega)$ as the space of square integrable functions over $\Omega$, where $\Omega \subset \mathbb{R}^d$ is an open, bounded and connected subset of a Euclidean domain of dimension  with a $C^2$ boundary. The standard inner product on $L^2(\Omega)$ will be denoted by $\langle \cdot , \cdot \rangle_{2}$, given by
	$
	\langle f, g\rangle_{2} = \int_{\Omega}  f(\mathbf{x})g(\mathbf{x})d\mathbf{x}
	$
	for each $f,g \in L^2(\Omega)$. The norm $\|\cdot\|_{2}$ on the space $L^2(\Omega)$ is defined as
	$
	\|f\|_{2} = \langle f, f\rangle^{1/2}_{2}
	$
	for all $f \in L^2(\Omega)$. We define the Sobolev space $H^1(\Omega)$ the set of $L^2$ functions with weak derivatives in $L^2(\Omega) $.
	The space $L^\infty((0,T);\Omega)^d$ is the set of essentially bounded vector fields. 
	The set $\mathcal{P}(\Omega)$ will denote the set of Borel probability measures on $\Omega$. 
	The space $W^{1,\infty}(\Omega)$ is the set of functions in $L^{\infty}(\Omega)$ with weak derivatives in $L^{\infty}(\Omega)$. If $f,g \in L^{\infty}(X)$, then $\|f\|_{2,g} := \int_X |f(x)|^2g(x)dx$ is the weighted $2$-norm with weight $g$. 
	
	In addition to this, we will need some mild assumptions on the activation function $\sigma:\mathbb{R}\rightarrow \mathbb{R}$.
	For this purpose, let us define the set of functions  
	\begin{align*}
	\mathcal{F}=\bigcup_{m\in \mathbb{Z}_+}\{ \sum_{i=1}^m \alpha_i \sigma(w_i^Tx+b_i) \ | \ \alpha_i \in \mathbb{R}, w_i \in \mathbb{R}^d, b_i \in \mathbb{R}\}.
	\end{align*}

	\begin{assumption}
		\label{asmp:neura}
		We make the following assumptions:
		\begin{enumerate}
			\item \textbf{(Regularity)} The activation function $\sigma$ is globally Lipschitz, that is, there exists $K>0$ such that 
			\begin{equation}
			|\sigma (x) - \sigma (y)| \leq K|x-y|, 
			\label{asmp:neura1}
			\end{equation}
			for all $x,y\in \mathbb{R}$.
			\item \textbf{(Density of superpositions)} The set of functions  $ \mathcal{F}$
			is dense in $C(\mathbb{R}^d;\mathbb{R})$ in the uniform norm topology on compact sets. Particularly, given a function $f \in C(\mathbb{R}^d;\mathbb{R})$, for each compact set $\Omega \subset \mathbb{R}$ and $\delta>0$, there exists a function $g \in \mathcal{F}$ such that 
			\[\sup_{x\in \Omega} |f(x) -g(x)| <\delta.\]     \label{asmp:neura2}
		\end{enumerate}
	\end{assumption} 
	
	Note that the set $\mathcal{F}$ is the set of arbitrarily wide neural networks. It is well-known that the Logistic function  and the ReLU function satisfy the density property, see~\cite{cybenko1989approximation,leshno1993multilayer}.
	\begin{align*}
	\mathcal{F}_d=\bigcup_{m\in \mathbb{Z}_+}\{ \sum_{i=1}^m A_i \Sigma(W_ix+B_i) \ | \ A_i , W_i \in \mathbb{R}^{d\times d}, B_i \in \mathbb{R}^d\},
	\end{align*}
	Given Assumption \ref{asmp:neura}, it is easy to see that the subset of vector-valued functions $ \mathcal{F}_d$ is dense in $C(\mathbb{R}^d;\mathbb{R}^d)$ in the uniform norm topology on compact sets. 
	\section{Analysis}
	\label{sec:analysis}
	In this section, we perform our analysis to demonstrate the expressivity of Neural SDEs by carrying out a controllability analysis. For ease of presentation the proofs are provided in the Appendix (Section \ref{sec:supp}). 
	
	We first investigate the class of data distributions that can be sampled using score matching approach presented in Section \ref{sec:prob}. Toward this end, we first note a continuity result with respect to the score matching loss. Particularly, the following proposition states that if we have a sequence of vector fields that is close to a reference vector field in the score matching loss \eqref{eq:scorlos}, then the solutions of the Fokker-Planck equations converge uniformly in the $L^2$ norm.

	\begin{proposition}
		\label{prop:scorapp}
		Let $\rho_0 \in L^\infty(\Omega)$ and $ \rho_d \in W^{1,\infty}(\Omega)$ be such that $\rho_d \geq  l$ for some constant $l>0$. Suppose that $\rho$ is the solution of \eqref{eq:genPde} corresponding to the vector field $V  = \nabla \log \rho^f_{T-\cdot}$, where $\rho^f$ is the solution of \eqref{eq:probodepdf}. Let $V^n$ be a sequence of vector fields such that $\|V^n\|_{\infty}$ is uniformly bounded and the score of $V^n$ with respect to $V$ tends to $0$, that is,
		\[\|V - V^n\|^2_{2,\rho^r} \rightarrow 0,\] where $\rho^r = \rho^f_{T-\cdot}$. Then we have that,
		\[\sup_{t \in [0,T]}\|\rho_t-\rho_t^n\|_2 \rightarrow 0 \]
	\end{proposition}
	
	The idea behind the proof is the following. We compute $ \|\rho^n_t -\rho_t\|^2_2$ by computing the derivative of the quantity. In general,  decrease in $\|V - V^n\|^2_{2,\rho^r}$ does not ensure that $ \|\rho^n_t -\rho_t\|^2_2$ is decreasing. However, it can be shown that $\rho^n_t$ remains bounded in the $\infty-$norm using Lemma \ref{lem:bnd}, which enables to get convergence of $\sup_{t \in [0,T]}\|\rho_t-\rho_t^n\|_2$.

	In the following we establish the class of data distributions that can be represented using the exact score, and derive bounds on the norm of the vector field as a function of the data distribution. This emphasizes how one can approximate a large class of distributions. It removes the Lipschitz requirement on score of the data distributions as made in \cite{chen2023sampling}. More importantly convergence is shown in the $L^2$ norm.
	
	\begin{lemma}
		\label{lem:scor}.
		Suppose $\Omega$ is convex. Let $\rho_0 = \rho_n$ be the uniform distribution on $\Omega$. Let $\rho_d \in W^{1,\infty}(\Omega) \cap \mathcal{P}(\Omega)$ and $\rho_d \geq l$ for some constant $l>0$.  Suppose $\rho_d$ is the initial condition of \eqref{eq:probodepdf}. Let $V =  \frac{\nabla \rho^{f}_{T-t}}{\rho^{f}_{T-t}} $ where $\rho_f$ is the solution of \eqref{eq:probodepdf}. Let $\rho$ be the solution of \eqref{eq:genPde} corresponding to the vector field $V$. Then $V \in L^{\infty}([0,T] \times \Omega)^d$ and
		\[\|\rho_T-\rho_d\|_2 \leq C_2 e^{\frac{\sqrt{T}}{l}}e^{-\lambda T} \|\rho_d - \rho_n \|_2\]
		\[ \|V\|_{\infty} \leq \frac{\|\nabla \rho_d\|_\infty}{l} \]
		where $\lambda >0$ depends only on $\Omega$.
		Hence, for every $\epsilon>0$ there exists $T>0$ large enough such that the solution $\rho_\epsilon$ of $\eqref{eq:genPde}$.
		\[\|\rho^{\epsilon}_T-\rho_d\|_2 \leq \epsilon \]
	\end{lemma}
	
	One of the key issues addressed in the proof is that, when implementing the score-matching algorithm \eqref{eq:approxrevsde}, one initializes from $\rho_0 = \rho_n$ the noise distribution, and not the distribution of the reverse process $\rho^f_T=\rho^r_0$, hence it is not necessary that sampling with the vector field $V$ would result in $\rho_d$, unless $V$ is contractive in some sense. We are able to use gradient estimates of the heat equation to ensure however that one can effectively sample from the data distribution $\rho_d$. Note that when $V_t$ is assumed to be Lipschitz one can get even stronger regularization of $\rho_t$. In this Lipschitz case, in fact solutions $\rho_t$ are in $H^1(\Omega)$, the set of functions with weak derivatives in $L^2(\Omega)$. See Proposition \ref{prop:Lip}.
	
	Next, we state our trajectory approximation result using limited width network addressing the control problem stated in Section \ref{sec:prob}.  A major difficulty in the limited width case is that one cannot approximate the reference vector field in any strong sense (uniform norm or $L^2$ norm) since the right hand side of \eqref{eq:neurtra} has very limited representation capability for each time instant. However, the system \eqref{eq:neurtra} still has sufficient controllability. The idea is that one can still achieve approximation by {\it weakly} approximating the vector fields in a time averaged sense (Proposition \ref{prop:weakapo}). This is a different approach to achieve controllability of \eqref{eq:genPde}, in comparison with a Lie bracket arguments used in \cite{tabuada2022universal} and the constructive strategy adopted in \cite{ruiz2023control}, for map approximations and density estimation, respectively. The approach we use is identical to the strategy used in \cite{elamvazhuthi2022neural} in the deterministic case. A key difference is that due to the regularizing effect of noise, and smoothness of the initial condition, one can achieve compactness of approximating trajectories (and hence, expressivity) of neural SDEs in the $L^2$ norm.
	
	\begin{theorem}
		\textbf{(Approximation of Trajectories)}
		\label{thm:trajapp}
		Suppose that Assumptions~\ref{asmp:neura} holds and $\rho_0 \in \mathcal{P}(\Omega) \cap H^1(\Omega)$.  Let $\rho$ be the solution of the ~\eqref{eq:genPde} corresponding to the vector field $V$. Additionally, suppose that $V \in C([0,T] \times \bar{\Omega} )^d$.
		Then  for every $\epsilon >0$, there exist  piecewise constant control inputs $A^\epsilon(\cdot), W^{\epsilon}(\cdot)$ and $B^\epsilon(\cdot)$, such that the corresponding weak solutions $\rho^\epsilon$ of~\eqref{eq:neurtra} satisfy
		\begin{equation}
		\sup_{t \in [0,T]} \| \rho_t^\epsilon-\rho_t \|_2 \leq \epsilon.
		\end{equation}
	\end{theorem}

	While the above theorem establishes the capability of limited width neural ODEs for trajectory approximation, it is not immediately clearly what are the class of distributions that can be represented using neural SDEs. Toward this end we note the following result on finding an exact vector field to interpolate between two densities. This result improves on Lemma \ref{lem:scor} by allowing $T=1$. It also improves over the result of \cite{tzen2019theoretical} for expressivity of neural SDEs. One can use the above theorem to conclude the expressivity for the limited width case. The idea behind the proof is that one use the {\it $L^2$-interpolation} between densities $\rho_t = (1-t)\rho_0 +t \rho_d$, and then we can construct a vector field $V = \frac{\nabla \rho_t}{\rho_t} -  \frac{\nabla \phi}{\rho_t}$, where $\phi$ is the solution of the Poisson equation
	\begin{equation}
	\Delta \phi = \rho_d -\rho_0.
	\end{equation}
	This achieves the controllability result. This construction is due to  \cite{moser1965volume}, and has been used in practice for generative modeling \cite{rozen2021moser}. We verify the boundedness of the vector field, under the weakest possible regularity assumptions on $\rho_0$ and $\rho_d$, and later apply the result to express densities using neural SDEs which have input dimension limited to $d$.
	\begin{proposition}
		\label{prop:exaccon}
		\textbf{(Exact Controllability)}
		Suppose that $\rho_n,\rho_d \in \mathcal{P}(\Omega) \cap C^1(\bar{\Omega})$, such that $\rho_0 \geq l$ and $\rho_d \geq l$ for some constant $l>0$. 
		Then there exists $V \in C([0,1] \times \bar{\Omega})^d$ and constant $C>0$ such that
		\begin{equation}
		\|V\|_{\infty} \leq 2C\frac{\max \{\|\nabla \rho_0\|_{\infty}, \|\nabla \rho_d\|_{\infty}\}}{l}
		\end{equation}
		where $C$ depends only on $\Omega$, and the solution $\rho$ of \eqref{eq:genPde} satisfies $\rho_0 =\rho_n$ $\rho_1=\rho_d$.
	\end{proposition}
	Comparing the statement of the above Proposition with Lemma \ref{lem:scor}, we note that we do not require convexity of the domain $\Omega$. Additionally, it is clear from the proof that one can relax the condition $\rho_0,\rho_d \in C^1(\bar{\Omega})$, by requiring they are only in $W^{1,\infty}(\Omega)$, and hence Lipschitz. However, then $V$ is not necessarily continuous anymore but only measurable. Nevertheless, it satisfies the same $L^{\infty}$ bounds. 
	
	Using the above proposition and Theorem \ref{thm:trajapp} on approximation of trajectories, the following result immediately follows.
	
	\begin{theorem}\textbf{(Approximate Controllability of limited width neural SDEs)}
		Suppose that Assumption \ref{asmp:neura} holds and $\rho_0,\rho_d \in \mathcal{P}(\Omega) \cap C^1(\bar{\Omega})$, such that $\rho_0 \geq l$ and $\rho_d \geq l$ for some constant $l>0$. 
		
		Then for every $\epsilon >0$, there exists piecewise constant control inputs  $A^\epsilon(\cdot), W^{\epsilon}(\cdot)$ and $B^\epsilon(\cdot)$, such that the corresponding solutions $\rho^\epsilon$ of the equation \eqref{eq:neurtra}, satisfy	\begin{equation}
		\|\rho^{\epsilon}_T - \rho_d\|_2  \leq \epsilon.
		\end{equation}
	\end{theorem}
	
	\section{Conclusion}
	We discussed the smoothening effect of noise in diffusion models, by considering the approximation properties of neural SDEs. Potential future directions could include investigating the role of noise in the forward process. For example, regularizing effect of noise makes the probability distribution of the Fokker-Planck equation strictly positive. Hence, while noise enables expressivity in a stronger norm, there is a potential trade-off, as it prevents sampling from distributions that  have supports that are disconnected and hence have unbounded score.
	\appendix

	\section{Supplementary Results}
	\label{sec:supp}
	In this section we present the some supplementary results. We will need some additional notation for the presentation. We equip the space $H^1(\Omega)$ with the usual Sobolev norms 
	$
	\|y\|_{H^1(\Omega)} = \Big( \|y\|^2_{2} + \sum_{i=1}^2 \left\| \frac{\partial y}{\partial x_i} \right\|^2_{2}\Big)^{1/2}
	$. The dual space of $X:=H^1(\Omega)$, denoted by $X^*$, is the space of bounded linear functionals on $H^1(\Omega)$. The space $L^2(0,T;Y)$ consists of all strongly measurable functions $u:[0,T] \rightarrow Y$ for which  $ \|u\|_{L^2(0,T;Y)} := \big ( \int^T_0 \|u(t)\|_Y^2 dt \big )^{1/2} ~<~ \infty $. We will say that a sequence $V^n \in L^{\infty}((0,T) \times \Omega)^d$ is weakly-$*$ converging to $V$, if $\int_0^T\int_{\Omega}V^n(x) \cdot \phi(x)dx \rightarrow  \int_0^T\int_{\Omega}V(x) \cdot \phi(x)dx$ for each $\phi \in L^1((0,T) \times \Omega)^d$.

	For the analysis in the paper, we will need a notion of solution for the PDE \eqref{eq:genPde}. Following \cite{evans2022partial}, we will use the notion of a {\it weak solution}. Given $g \in L^2(\Omega)$ and 
	a vector-field $V : [0,T] \times \mathbb{R}^d \rightarrow \mathbb{R}^d$, we will say that $\rho$ is a weak solution
	of the Fokker Planck equation \eqref{eq:genPde} if 
	\begin{eqnarray}
	\rho \in \{ u \in L^2(0,T;H^1(\Omega)), \dot{u}  \in L^2(0,T;X^*) \}\\
	\int_0^T \langle \partial_t\rho , \phi_t \rangle_{X,X^*}dt + \int_0^T B[\rho_t,\phi,t] dt = 0
	\end{eqnarray}
	for each function $ \phi \in H^1(\Omega)$ for almost every time $0\leq t \leq T$ and $y(0) = g$, where $B:H^1(\Omega) \times H^1(\Omega) \rightarrow \mathbb{R}$ is the bilinear form given by,
	
	\[B[u, v,t] = \int_{\Omega} \nabla \rho(x) \cdot \nabla\phi (x)   dx + \int_{\Omega} V_t(x) \rho(x) \cdot \nabla\phi (x)   dx\]
	
	for each $u,v \in H^1(\Omega)$ and all $t \in [0,T]$.
	
	Given this notion of solution we have this following classical existence result. In addition, we can establish continuity of solutions with respect to the initial conditions.
	
	\begin{proposition}
		\label{prop:ex}
		Suppose that $\rho_0 \in L^2(\Omega)$ and $V \in L^\infty((0,T) \times \Omega)^d$. Then there exists a unique $\rho \in C([0,T];L^2(\Omega))$ that is a weak  solution of the Fokker Planck equation. We have
		\begin{equation}
		\int_0^T\|\rho_t\|^2_{H^1(\Omega)}dt <C_1,~~ \sup_{t \in [0,T]} \|\rho_t\|_2 \leq C_2
		\end{equation}
		where the constant $C_1,C_2>0$ depend only on $\|V\|_{\infty}$. Moreover, the solution $\rho(t)$ is continuous with respect to $\rho_0$ in the $L_2$ norm, for every $t \in [0,T]$. Particularly, suppose $\rho_1, \rho_2$ are solutions \eqref{eq:genPde} for two different initial conditions $\rho_{1,0}, \rho_{2,0} \in L_2(\Omega)$. If $V \in C([0,T] \times \bar{\Omega})^d$, then there exists $C_2>0$ such that 
		\[ \|\rho^1_t -\rho^2_t \|^2_2 \leq C_2 e^{\int_0^t\|V_\tau\|_{\infty} d\tau} \|\rho_{1,0}- \rho_{2,0}\|^2_2\]
		where the constant $C_2$ depends only on $\|V\|_{\infty}$.
	\end{proposition}
	\begin{proof}
		The existence of solution is classical in PDE theory. See \cite{wloka1987partial}. We only establish the continuity of solutions. Let $ \rho^1_t - \rho^2_t$ be a test function. Then we have that 
		\begin{eqnarray}
		\int_0^T \langle \partial_t \rho^1_t, \rho^1-\rho^2 \rangle_{X,X^*}dt + \int_0^T B[\rho^1_t,\rho^1_t \nonumber  -\rho_t^2,t]dt \\ -\int_0^T \langle \partial_t \rho^1_t, \rho^1-\rho^2 \rangle_{X,X^*} dt + \int_0^T B[\rho^2_t,\rho^1_t -\rho^2_t,t]  dt =0  
		\end{eqnarray}
		for all $t \in [0,T]$.
		\begin{eqnarray}
		\int_0^T \langle \partial^1_t \rho^1_t, \rho^1-\rho^2 \rangle_{X,X^*}dt + \int_0^T B[\rho^1_t,\rho^1_t \nonumber  -\rho_t^2,t]dt \\ -\int_0^T \langle \partial_t \rho^1_t, \rho^1-\rho^2 \rangle_{X,X^*} dt - \int_0^T B[\rho^2_t,\rho^1_t -\rho^2_t,t]  dt =0  
		\end{eqnarray}
		This implies 
		\begin{eqnarray}
		\int_0^T \frac{d}{dt}\|\rho^1_t - \rho^2_t \|^2_2dt + \int_0^T B[\rho^1_t,\rho^1_t \nonumber  -\rho_t^2,t] -\int_0^T B[\rho^1_t,\rho^1_t \nonumber  -\rho_t^2,t]dt  =0
		\end{eqnarray}
		\begin{eqnarray}
		\|\rho^1_t - \rho^2_t \|^2_2-\|\rho^1_0 - \rho^2_0 \|^2_2+ \int_0^T B[\rho^1_t,\rho^1_t \nonumber  -\rho_t^2,t] -\int_0^T B[\rho^1_t,\rho^1_t \nonumber  -\rho_t^2,t]dt  =0
		\end{eqnarray}
		\begin{eqnarray}
		\|\rho^1_t - \rho^2_t \|^2_2-\|\rho^1_0 - \rho^2_0 \|^2_2+ \int_0^T B[\rho^1_t,\rho^1_t \nonumber  -\rho_t^2,t]dt -\int_0^T B[\rho^1_t,\rho^1_t \nonumber  -\rho_t^2,t]dt  =0
		\end{eqnarray}
		Define $ e_t := \rho^1_t - \rho^2_t $ for all $t \in [0,T]$.
		\begin{eqnarray}
		\|e_t \|^2_2 &=& \|e_0 \|^2_2 - \int_0^T \|\nabla e_t\|^2_2 dt  + \int_0^T \|V_te_t \cdot \nabla e_t\|^2_2 dt \nonumber \\
		&\leq & \|e_0 \|^2_2 + \int_0^T \|\nabla e_t\|^2_2 dt  - \int_0^T \|V_t\|_{\infty} \|e_t \cdot \nabla e_t\|^2_2 dt
		\end{eqnarray}
		Using Cauchy-Schwarz inequality and Cauchy's inequality \cite{evans2022partial}[Page 622, 624], for every $\epsilon>0$
		\begin{eqnarray}
		\|e_t \|^2_2 &\leq & \|e_0 \|^2_2 - \int_0^T \|\nabla e_t\|^2_2 dt  + \frac{1}{2\epsilon}\int_0^T \|V_t\|_{\infty} \| e_t\|^2_2 dt  + \frac{\epsilon}{2}\int_0^T \|V_t\|_{\infty}\|  \nabla e_t \|^2_2 dt  \nonumber \\
		&\leq & \|e_0 \|^2_2 - \int_0^T \|\nabla e_t\|^2_2 dt  + \frac{1}{4\epsilon}\int_0^T \|V_t\|_{\infty} \| e_t\|^2_2 dt + \epsilon \|V\|_{\infty} \int_0^T \| \nabla e_t \|^2_2 dt  \nonumber 
		\end{eqnarray}
		Setting $\epsilon = \frac{1}{\|V\|_{\infty}}$, we can conclude that,
		\begin{eqnarray}
		\|e_t \|^2_2 &\leq & \|e_0 \|^2_2 + \frac{1}{4\epsilon}\int_0^T \|V_t\|_{\infty}\|e_t \|^2_2 dt  \nonumber
		\end{eqnarray}
		Now the result follows from Gronwall's inequality. The bounds on the solutions can be computed in a similar way by computing $\langle \partial_t\rho_t , \rho_t \rangle_{X,X^*}$. 
	\end{proof}
	Next, we observe a regularity property of solutions. This property enables us to get compactness of trajectories in $C([0,T];L^2(\Omega))$, which will play a key role in constructing approximating trajectories using neural networks of limited width. 
	\begin{proposition}
		\label{prop:Lip}
		(\textbf{Extra regularity})
		Suppose that $\rho_0 \in H^1(\Omega)$ and $V \in C([0,T]\times \bar{\Omega})$ and $V_t$ be uniformly Lipschitz. That is, $|V_t(x)-V_t(y)| \leq K|x-y|$ , for all $x, y \in \bar{\Omega}$, for some constant $K>0$ independent of $t \in [0,T]$. Then we have that the unique weak solution $\rho$ that lies in $L^\infty([0,T];H^1(\Omega))$  and we have the estimate
		\begin{equation}
		\sup_{t \in [0,T]} \|u(t)\|_{H^1(\Omega)} <C
		\end{equation}
		where the constant $C>0$ depends only on $K$.
	\end{proposition}
	\begin{proof}
		This proof is a very minor adaptation of proof of \cite{evans2022partial}[Chapter 7, Theorem 5] on improved regularity results of parabolic PDEs. The only difference is that we consider a more general boundary condition than the {\it Dirichlet boundary condition} considered in \cite{evans2022partial}. We only verify that all the computations extend to our case in a similar way.
		Let $\partial_t \rho_t$ be a test function. Since $\rho$ is a weak solution  we have that 
		
		\[\int_0^T \langle \partial_t \rho,\partial_t \rho \rangle_{X,X^*}dt + \int_0^T B[\rho_t,\partial_t\rho_t,t]dt  = 0\]
		
		for all $t \in [0,T]$.
		Since $V_t$ is uniformly Lipschitz, $V_t \in W^{1,\infty}(\Omega)$ for all $t$ and there exists $K'>0$ such that $\|V_t\|_{W^{1,\infty}} < K'$ by \cite{evans2022partial}[Chapter 5, Theorem 4].
		Then we have that 
		\begin{eqnarray}
		\int_0^T   \| \partial_t \rho_t \|^2_2dt + \int_0^T  \frac{d}{dt}\|\nabla \rho \|_2^2dt  = -\int_0^T <\nabla\cdot (V_t\rho), \partial_t \rho_t>dt \\
		\leq \frac{1}{4\epsilon} \int_0^T \|V_t\|_{W^{1,\infty}}\|\nabla \rho_t\|^2_{2}dt + \epsilon \int_0^T \|\partial_t\rho_t \|^2_2dt \nonumber
		\end{eqnarray}
		by Cauchy-Schwarz inequality and Cauchy's inequality
		Let $\epsilon =1$, we have that
		
		\begin{eqnarray}
		\int_0^T  \frac{d}{dt}\|\nabla \rho_x\|_2^2 dt  \leq 
		\frac{1}{4}\int_0^T \|V_t\|_{W^{1,\infty}}\|\nabla \rho_t\|^2_2dt  \nonumber
		\end{eqnarray}
		Using  the fundamental theorem of calculus we conclude that,
		\begin{eqnarray}
		\|\nabla \rho_T\|_2^2- \|\nabla \rho_0\|_2^2  \leq \frac{1}{4} \int_0^T \|V_t\|_{W^{1,\infty}}\|\nabla \rho\|^2_2 dt \nonumber
		\end{eqnarray}
		Using Gronwall's inequality we get, 
		\begin{eqnarray}
		\sup_{t \in [0,T]}\|\nabla \rho_t\|_2^2   \leq \frac{ \|\nabla \rho_0\|_2^2}{4} e^{\int_0^T \|V_t\|_{W^{1,\infty}}dt} \nonumber
		\end{eqnarray}
		This in combination with the bound on $\sup_{t\in [0,T]}\|\rho_t\|_2^2$ in Proposition \ref{prop:ex} gives the result.
		So far we have assumed that the computations are admissible since we are assuming that $\partial_t \rho_t$ is in $L^2(\Omega)$, which is not in general true for an arbitrary weak solution. The proof can be completed by constructing finite dimensional Galerkin approximations of the PDE \eqref{eq:genPde} and taking converging sub-sequences of the approximations. See \cite{evans2022partial}[Chapter 7, Proof of Theorem 5]. 
	\end{proof}
	
	The following Lemma establishes continuity of solutions with respect to the vector fields.
	\begin{lemma}
		\label{lem:uniapp}
		Let $\rho_0 \in L^2(\Omega)$ and $V \in C([0,T] \times \bar{\Omega})^d$ be such that $\rho$ is the solution of \eqref{eq:genPde}. Let $V^n$ be a sequence of vector fields such that
		\[\|V - V^n\|_{\infty} \rightarrow 0\] 
		for some function that is uniformly bounded from below. Then 
		\[\sup_{t \in [0,T]}\|\rho_t-\rho_t^n\|_2 \rightarrow 0 \]
	\end{lemma}
	\begin{proof}
		Let $e^n = \rho^n - \rho$. Computing,
		\begin{eqnarray}
		\frac{d}{dt}\|e^n_t\|_2^2 & = & \langle \partial_t e^n_t , e^n_t \rangle_{X,X^*} \\
		& = & -\langle \nabla e_t^n, \nabla e_t^n \rangle_2 + \langle V_t \rho_t - V^n_t \rho_t^n, \nabla e_t^n \rangle_2 \nonumber \\
		& = & -\langle \nabla e_t^n, \nabla e_t^n \rangle_2 + \langle V_t \rho_t - V^n_t \rho_t, \nabla e_t^n \rangle_2+\langle V^n_t \rho_t - V^n_t \rho_t , \nabla e_t^n \rangle_2
		\end{eqnarray}
		
		Once again applying Cauchy-Schwarz inequality and Cauchy's inequality for every $\epsilon>0$
		\begin{eqnarray}
		\frac{d}{dt}\|e^n_t\|_2^2 & \leq & 
		-\|\nabla e_t^n \|_2^2 +  \frac{4}{\epsilon}\| V \rho_t - V^n_t \rho_t \|^2_2  + \frac{4}{\epsilon} \| V^n_t \rho_t - V^n_t \rho_t^n \|^2_2 +  \epsilon \|\nabla e_t^n \|_2^2
		\end{eqnarray}
		Setting $\epsilon = 1$ we can conclude that,
		\begin{eqnarray}
		\frac{d}{dt}\|e^n_t\|_2^2 & \leq & 
		4\| V_t \rho_t - V^n_t \rho_t \|^2_2  + 4\| V^n_t \rho_t - V^n_t \rho_t^n \|^2_2 \nonumber \\
		& \leq & 
		4\| V_t \rho_t - V^n \rho_t \|^2_2  +  4\| V^n \rho_t - V^n \rho_t^n \|^2_2 \label{eq:ineq1}
		\end{eqnarray}
		Using the hypothesis of the theorem we can conclude that
		\begin{eqnarray}
		\frac{d}{dt}\|e^n_t\|_2^2 & \leq & 
		4\| V - V^n  \|^2_{\infty} \|\rho_t\|^2_{2}  +   4\|V^n\|_{\infty}\|e^n_t \|^2_2 
		\end{eqnarray}
		Using Gronwall's inequality we get that,
		\begin{eqnarray}
		\|e^n_t\|^2_t \leq 4 \| V - V^n  \|_{\infty} e^{4\| V^n\|_{\infty}t}
		\end{eqnarray}
		This concludes the result.
	\end{proof}
	
	For the score matcing approximations of vector fields, we will need to be able to find a bounds on the $\infty$-norm of the solutions of \eqref{eq:genPde}. This is stated in the following Lemma.
	\begin{lemma}
		\label{lem:bnd}
		Let $f \in W^{1,\infty}(\Omega)$ and $f \geq l$ for some $l >0$. Consider the PDE
		\begin{eqnarray}
		&y_t = \Delta y - \nabla \cdot(\frac {\nabla f(\mathbf{x})}{f(\mathbf{x})} ~ y) &~~ in  ~~ \Omega \times [0,T] \nonumber \\ 
		&y(\cdot,0) = y^0 &~~ in ~~ \Omega \nonumber \\ 
		&\mathbf{n} \cdot ( \nabla y - \frac {\nabla f(\mathbf{x})}{f(\mathbf{x})}y ) =0 &~~ in ~~ \partial \Omega \times [0,T],   
		\label{eq:cllp2}
		\end{eqnarray}
		
		If $\|  \frac{y_0}{f}  \|_{\infty} \leq C$, then the solution $y$ satisfies $\|  \frac{y_t}{f}  \|_{\infty} \leq C$ for all $t \in [0,T]$.
	\end{lemma}
	\begin{proof}
		The proof follows verbatim the proof of the last statement of \cite{elamvazhuthi2018bilinear}[Corollary IV.2], where the statement has been proved for the PDE $y_t =\Delta  (\frac {y_t}{f(\mathbf{x})})$
	\end{proof}
	
	Now, we are ready to provide proof of Proposition \ref{prop:scorapp}.
	\begin{proof}
		Let $e^n = \rho^n - \rho$. From \eqref{eq:ineq1} we have that
		\begin{eqnarray}
		\frac{d}{dt}\|e^n_t\|_2^2 & \leq & 
		4\| V_t \rho_t - V^n \rho_t \|^2_2  +  4\| V^n \rho_t - V^n \rho_t^n \|^2_2 \nonumber \\
		& \leq &  4\| V_t -V^n_t \|^2_{2} \|\rho_t\|_{\infty} +  4\| V^n \rho_t - V^n \rho_t^n \|^2_2 \nonumber \\
		& \leq &  \frac{4}{l}\| V_t -V^n_t \|^2_{2,\rho^r_t} \|\rho_t\|_{\infty} +    4\|V^n\|_{\infty}\|e^n_t \|^2_2  \label{eq:ineq2}
		\end{eqnarray}
		where we have used the fact that since the $\rho^f_0 = \rho_d \geq l $ by the assumption, $\rho^f_{T-t} = \rho_r^t \geq l$ for all $ t\in [0,T]$, since $\rho^f_t$ is the solution of the equation and hence so is $\rho^f_t -l$, and solutions of the heat equation preserve non-negativity.
		Next, we need a bound on $\|\rho_t\|_{\infty}$ in \eqref{eq:ineq2}. Since, $\rho^f$ is the solution of the heat equation, we know that $\rho^f \in C^{1,2}((0,T]\times\bar{\Omega})$, which implies that $\rho^r \in C^{1,2}([0,T) \times \bar{\Omega})$ \cite[Theorem 3.1]{bertoldi2004gradient}. Therefore, $\rho$ is the solution of the equation 
		\begin{eqnarray}
		\frac{\partial \rho}{\partial t} = \Delta \rho -\nabla \cdot \big(\frac{\nabla \rho^r_t}{\rho^r_t}\rho \big)    = \nabla \cdot (\rho^r_t \nabla (\frac{1}{\rho^r_t} \rho))
		\end{eqnarray}
		We can consider a piecewise constant approximation of $V$ of the form,
		\[V^m_t = V_\frac{iT}{m} = \frac{\nabla \rho_\frac{iT}{m}^r}{\rho^r_\frac{iT}{m}} ~~ t \in [\frac{iT}{m},\frac{(i+1)T}{m}). \]
		for $ i = 0, ..., m-1$.
		Fix $\epsilon>0$ small enough. Since $V$ is continuous on any interval $[0,T-\epsilon]$, $V^m$ uniformly converges to $V$ on $[0,T-\epsilon]$. Consider the solutions of the equation 
		\begin{eqnarray}
		\frac{\partial \rho^m}{\partial t} = \Delta \rho^m -\nabla \cdot \big(V^m_t\rho^m \big)    
		\end{eqnarray}
		The solutions of the above equation can be constructed in piecewise way by solving a sequence of autonomous linear equations.
		of the form 
		\begin{eqnarray}
		\frac{\partial \rho^m}{\partial t}  = \nabla \cdot (f^i \nabla (\frac{1}{f^i} \rho^m))
		\label{eq:poteq}
		\end{eqnarray} 
		where $f^i = \rho^r_\frac{iT}{m}$. By applying Lemma \ref{lem:bnd} across each time interval $[\frac{iT}{m},\frac{(i+1)T}{m})$ we can conclude that  if $\|f^i \rho_0 \|_{\infty} \leq \|f^i\| \|\rho_0 \|_{\infty}<  \|\rho^r\|_{\infty} \|\rho_0 \|_{\infty}<C$, then $\|f^i \rho^m_t\|_{\infty}<C$ for all $ t \in  t \in [\frac{iT}{m},\frac{(i+1)T}{m})$. This implies $\|\rho^m_t\|_{\infty}$ with bound independent of $m$. is uniformly bounded. We can make the same conclusion for $\rho$ by taking the limit $V^m \rightarrow V$ and applying Lemma \ref{lem:uniapp}, we have that $\rho^m$ are uniformly converging to $\rho$ in $C([0,T];L^2(\Omega))$. Hence same uniform bound holds for $\|\rho_t\|_{\infty}$. Since $\epsilon>0$ was arbitrary, this bound holds over the entire time interval $[0,T]$. Combining this result with the estimate of \eqref{eq:ineq2}, the result follows. Since $\|V^n\|_\infty$ are uniformly bounded. The solutions $\rho^m$ are bounded in 
	\end{proof}
	
	Proof of Lemma \ref{lem:scor}
	\begin{proof}
		It is a classical result that there exist constants $M,\lambda >0 $ for which
		\[\|\rho^f_t- \rho_n\|_2 \leq M e^{-\lambda t} \|\rho_0 - \rho_n\|_2 ~~~\forall t \geq 0 \]
		The solution to the heat equation can be represented by a semigroup of operators $(T(t))_{t \geq 0}$, as $\rho_f(t)= T(t) \rho_0 $. Additionally, it is known that the heat equation with the Neumann condition $\rho^f_t$ is known to be a contraction in $L_{\infty}$ norm. That is $\|T(t)\rho_0\|_{\infty} \leq \|\rho_0\|_{\infty}$ for all $t \in [0,\infty)$. The gradient operator and the heat semigroup commute. Hence, $\|\nabla \rho^f_t\|_{\infty} = \| T(t)\nabla \rho _0\|_{\infty} \leq \|\nabla \rho_0\|$ for all $t \in [0,\infty) $. Additionally, we also know that if $\rho_d > c $ for some positive constant $l>0$, then $\rho^f_t \geq l$ for all $ t \in [0,\infty)$, since it is the solution of the heat equation. This implies that $\|V\|_{\infty}  = \|\frac{\nabla \rho^f_{T-t}}{\rho_f}\|_{\infty}\leq   \frac{\|\nabla \rho_0\|_{\infty}}{c}$. Note that $\rho_f(T-t)$ is the solution of \eqref{eq:genPde} for $V$ with initial condition $\rho_0 = \rho_f(T)$. We know from \cite[Corollary 5.1]{bertoldi2004gradient} that there exist constants $K>0$ such that 
		\[\|\nabla\rho_t\|_{\infty} \leq \frac{K}{ \sqrt{t}}\|\rho_0\|_{\infty} \]
		By continuity of solutions $\rho_t$ with respect to the initial condition we know from Proposition that 
		\begin{eqnarray}
		\|\rho_T - \rho_d\|_2 & \leq & C_2e^{ \int_0^T\|V_t\|_{\infty}dt} \|\rho(0) - \rho_f(T)\|_2 \nonumber \\
		&=& C_2e^{K\sqrt{T}}\|\rho_{n} - \rho_f(T)\|_2 \nonumber \\
		&\leq & C_2e^{K\sqrt{T}-\lambda T} \|\rho_d - \rho_{n}\|_2 
		\end{eqnarray}
	\end{proof}
	In the following proposition, we improve on the statement of Lemma \ref{lem:uniapp}, by showing that one can achieving approximation of trajectories of \eqref{eq:genPde} by {\it weakly} approximating the trajectories, instead of in the uniform sense. A key role is played by the regularity statement of Proposition \eqref{prop:Lip}, which gives us compactness of trajectories.
	\begin{proposition}
		\label{prop:weakapo}
		Let $\rho_0 \in H^1(\Omega)$. Let $V^n \in L^\infty((0,T) \times \Omega)^d$ be a sequence of piecewise constant in time vector-field with uniform Lipschitz constant converging weakly-* to $V$ in $L^\infty((0,1) \times \Omega)^d$. Suppose $\rho^n$ and $\rho$ are weak solutions of the Fokker-Planck  equation \eqref{eq:genPde}, corresponding to the vector fields $V^n$ and $V$, respectively. Then $\rho^n$ converges to $\rho$ in $C([0,T];L_2(\Omega))$.
	\end{proposition}
	\begin{proof}
		Since $V^n$ are uniformly Lipschtz continuous in space, according to Proposition \ref{prop:Lip} the solutions are bounded in 
		\begin{equation}
		W = \{ u \in L^{\infty}(0,T;X), \dot{u} \in L^{2}(0,T;X^*)\}.
		\end{equation}
		Therefore, by Aubins-Lions lemma \cite[Corollary 4]{simon1986compact} we have compactness of trajectories, and there exists a subsequence of $(\rho^n)_{n=1}^\infty$ such that, again denoted by  $(\rho^n)_{n=1}^\infty$, such that $\rho^n $ converges to some $\hat{\rho} $ in $C([0,T];L_2(\Omega))$. 
		Fix $\phi \in H^1(\Omega)$. We can also conclude that
		\begin{eqnarray}
		\langle \partial_t \rho^n , \phi \rangle_{X,X^*} \rightarrow \langle \partial_t \hat{\rho} , \phi \rangle_{X,X^*} 
		\end{eqnarray}
		We know that $\sup_{t \in [0,T]} \|\rho^n\|_{H^1}$ is uniformly bounded in from Proposition \ref{prop:Lip}. Therefore, it follows that we can extract a subsequence such that,
		\begin{eqnarray}
		\int_0^T \int_{\Omega} \nabla  \rho^n_t(x) \cdot \nabla \phi(x)dx \rightarrow \int_0^T \int_{\Omega} \nabla  \hat{\rho}_t(x) \cdot \nabla \phi(x)dxdr\nonumber 
		\end{eqnarray}
		Lastly, we consider
		\begin{eqnarray}
		| \int_0^T \int_{\Omega}  V^n \rho^n_t(x) \cdot \nabla \phi(x)dx -  \int_0^T \int_{\Omega}  V \hat{\rho}(x) \cdot \nabla \phi(x)dxdt| \nonumber \\
		\leq | \int_0^T \int_{\Omega}  V^n \rho^n_t(x) \cdot \nabla \phi(x)dx -  \int_0^T \int_{\Omega}  V^n \hat{\rho}_t(x) \cdot \nabla \phi(x)dxdt| \nonumber \\
		| \int_0^T \int_{\Omega}  V^n \hat{\rho}_t(x) \cdot \nabla \phi(x)dx -  \int_0^T \int_{\Omega}  V \hat{\rho}_t(x) \cdot \nabla \phi(x)dxdt|
		\end{eqnarray}
		The first term in the above equality converges to zero, because $\rho^n$ converges to $\hat{\rho}$ in $C([0,T];L_2(\Omega))$. The second term converges to zero because $\rho  \nabla \phi (x) \in L^1( (0,1) \times \Omega)$ and $V^n$ converges to $V$ in $L^\infty((0,T) \times \Omega)$ in the weak-* topology. 
		All these convergences imply that
		$\hat{\rho}$ is a solution of the Fokker-Planck equation \eqref{eq:genPde} corresponding to the vector-field $V$. By uniqueness of solution $\hat{\rho} =\rho $, corresponding to the vector field $V$.
	\end{proof}
	
	The previous Proposition states that weak approximation of vector fields can give us strong convergence of trajectories. Now we show, that arbitrarily wide neural-network valued vector fields can be approximated in the weak-* sense, by neural networks with $d-$ dimensional widths.

	\begin{theorem}
		\label{thm:weaapp}
		Let $A_i,W_i \in \mathbb{R}^{d\times d}, B_i \in \mathbb{R}^d$ be weight parameters for $i= 1,...,m$. For each $N \in \mathbb{Z}_+$. Let $Q^N$ be a $\frac{T}{N}$-periodic vector field defined by 
		\begin{equation}
		\label{eq:defosc}
		Q^N_{t+\frac{nT}{N}}(x) =  m A_i\Sigma(W_ix+B_i), ~~ t \in [\frac{iT}{mN},\frac{(i+1)T}{mN}), 
		\end{equation}
		for all $n \in \{0,...,N-1\}$, $i \in \{ 0,1,...,m-1\}$ and $x \in \mathbb{R}^d$. Then, $ Q^N$ weakly-* converges to  $\sum_{i=1}^m A_i\Sigma(W_ix+B_i)$  in $L^\infty( (0,T) \times \Omega)^d$. 
	\end{theorem}
	\begin{proof}
		Let $\phi \in C([0,T] \times \bar{\Omega})^d$. Consider the integral 
		\[\int_{\Omega} \int_0^TQ_t^n(x) \cdot \phi_t(x)dtdx \]
		We know that  $Q^N(\cdot, x)$ weakly converges to $\sum_{i=1}^MA_i\sigma(W_ix+B_i^n)$, for each $x \in \mathbb{R}^d$ in $L^{p}(0,T)$ by \cite[Theorem 8.2]{chipot2009elliptic}, for any $p>1$. Since $Q^N$ are uniformly bounded, this also implies that they are weakly-* converging in $L^{\infty}(0,T)$.
		This implies that $\int_{\Omega} \int_0^T V_t^n(t,x)\phi(t,x)$ converges to $\int_{\Omega} \int_0^TW_t(t,x)\phi(t,x)$ for each $x \in \Omega$. Then it follows from the dominated convergence theorem that 
		\[\int_{\Omega} \int_0^TW_t^n(x) \cdot \phi_t(x)dtdx \rightarrow  \int_{\Omega} \int_0^TQ_t(x) \cdot \phi_t(x)dtdx \]
		Since the set of continuous functions is dense in $L^1(0,T \times \Omega )^d$, the result also holds true if $\phi \in L^1(0,T \times \Omega )^d$ and the result follows.
	\end{proof}
	
	Using the last two results we are able to provide the proof of Theorem \ref{thm:trajapp}.
	\begin{proof} 
		We consider approximations of $V$ as done in proof of Lemma \ref{lem:uniapp}. Define
		\[V^m_t = V_\frac{iT}{m} =  ~~ t \in [\frac{iT}{m},\frac{(i+1)T}{m}). \]
		Since $V$ is continuous and defined on a compact set, it is uniformly continuous, and hence \\ $\lim_{m\rightarrow. \infty }\|V -V^m\|_{\infty} = 0$. This implies that the solutions $\rho^m$ of \eqref{eq:genPde} corresponding to the vector fields $V^m$ uniformly converge to $\rho$ in $C([0,T];L^2(\Omega))$. Given the fact that $V^m$ is piecewise constant in time, we can approximate $V^m$ using a sequence $V^{m,n}$ such that $V^{m,n}_t \in \mathcal{F}_d$ for all $t\in [0,T]$ and by Assumption \ref{asmp:neura}, this is always possible. This implies that the sequence of solutions $\rho^{m,n}$ of \eqref{eq:genPde} corresponding to the vector fields $V^{m,m}$ uniformly converge to $\rho^m$ in $C([0,T];L^2(\Omega))$. Next, using Theorem \ref{thm:weaapp}, we a sequence of vector fields that can be represented by vector fields of the form  $A(t)\Sigma(W(t)x+B(t))$ that are weakly-* converging to $V^{m,n}$. Due to Assumption \ref{asmp:neura}, the vector fields $A(t)\Sigma(W(t)x+B(t))$ are uniformly Lipschitz. Hence, the result follows from Proposition \ref{prop:weakapo}.
	\end{proof}
	\begin{proof}\textbf{of Proposition \ref{prop:exaccon}}
		We first construct a $V \in C([0,1]\times \bar{\Omega})^d$ that achieves the (exact) controllability/expressiblity result. This is just a minor extension of the idea from \cite{elamvazhuthi2018bilinear}[Theorem IV.14]. The only difference being we reduce the differentiability requirement. 
		Let $\rho_t = (1-t)\rho_0 +t \rho_d$ for each $t \in [0,T]$. Let $\phi_t = \Delta^{-1} (\rho_d - \rho_0)$ for each $t \in [0,T]$. Let $V = \frac{\nabla \rho}{\rho} - \frac{\nabla \phi}{\rho}$, where $\Delta$ is the Laplacian with Neumann boundary condition. Since $\rho_d - \rho_0 \in L^{\infty}(\Omega) $, it follows from \cite[Theorem 2.4.27]{grisvard2011elliptic} that $\phi \in W^{2,p}(\Omega) $ for every $p > 1$. Hence, it follows from Morrey's inequality \cite[Theorem 11.34]{leoni2017first} that there exists a constant $C,C' >0$ such that $\|\Delta^{-1}(\rho_d - \rho_0)\|_{W^{2,p}} \leq C \|(\rho_d - \rho_0)\|_p \leq C' \|(\rho_d - \rho_0)\|_{\infty} $. From these inequalities we can conclude that,
		\begin{eqnarray}
		\|V\|_{\infty} &\leq& \frac{\max \{\|\nabla \rho_0\|_{\infty}, \|\nabla \rho_d\|_{\infty}\|\}}{c} + C'\frac{\max \{\|\nabla \rho_0 -\nabla \rho_d\|_{\infty}\}}{c} \nonumber \\
		&\leq& 2C'\frac{\max \{\|\nabla \rho_0\|_{\infty}, \|\nabla \rho_d\|_{\infty}\}}{c}.
		\end{eqnarray} 
		This concludes the bound.
		It can be checked that 
		\[\partial_t \rho_t = \rho_d - \rho_0 = \Delta \phi = \Delta \rho -\nabla \cdot (V \rho).\]
		Hence, $V$ exactly transfers the solution of \eqref{eq:genPde} from $\rho_0$ to $\rho_d$, and $\rho$ is the unique weak solution of \eqref{eq:genPde}.
	\end{proof}
	
	\bibliography{ref}
	
\end{document}